\documentclass[12pt,onecolumn,journal]{IEEEtran}

\usepackage{natbib}
\usepackage{caption}

\usepackage{bbm}
\usepackage{bm}
\usepackage{graphicx}
\usepackage{subcaption}
\usepackage{epstopdf}

\usepackage{amsmath}
\usepackage{amsfonts}
\usepackage{amssymb}
\usepackage{amsthm}
\usepackage{mathtools}
\usepackage{bm}

\newtheorem{theorem}{Theorem}

\newtheorem{lemma}{Lemma}
\newtheorem{corollary}{Corollary}

\newtheorem{remark}{Remark}

\usepackage{algorithm}
\usepackage{algorithmic}
\usepackage{hyperref}
\usepackage{setspace}
\usepackage{xcolor}

\begin{document}
\newcommand{\NAM}{\textsc{LCC-UCB}}
\renewcommand{\cite}[1]{\citep{#1}}
%

%

\title{Multi-Agent Multi-Armed Bandits with Limited Communication}

\author{Mridul Agarwal, Vaneet Aggarwal, Kamyar Azizzadenesheli \thanks{The authors are with Purdue University, West Lafayette IN 47907, USA, email:\{agarw180,vaneet,kamyar\}@purdue.edu.}}

\maketitle
\begin{abstract}

We consider the problem where $N$ agents collaboratively interact with an instance of a stochastic $K$ arm bandit problem for $K \gg N$. The agents aim to simultaneously minimize the cumulative regret over all the agents for a total of $T$ time steps, the number of communication rounds, and the number of bits in each communication round. 
We present Limited Communication Collaboration - Upper Confidence Bound (\NAM), a doubling-epoch based algorithm where each agent communicates only after the end of the epoch and shares the index of the best arm it knows. With our algorithm, \NAM, each agent enjoys a regret of $\tilde{O}\left(\sqrt{({K/N}+ N)T}\right)$, communicates for $O(\log T)$ steps and broadcasts $O(\log K)$ bits in each communication step. We extend the work to sparse graphs with maximum degree $K_G$, and diameter $D$ and propose \NAM-GRAPH which enjoys a regret bound of $\tilde{O}\left(D\sqrt{(K/N+ K_G)DT}\right)$. Finally, we empirically show that the \NAM\ and the \NAM-GRAPH algorithm perform well and outperform strategies that communicate through a central node. 

\end{abstract}

\vspace{-.25in}
\section{Introduction}

We consider a setup where $N$ agents connected over a network, interact with a multi armed bandit (MAB) environment \cite{lattimore2020bandit}. The agents aim to collaborate with other agents in the network to minimize their regret. The agents also aim to reduce the number of messages and the size of messages communicated with others. Consider a case of an e-commerce company serving its users by recommending its vast number of items through multiple servers for quick response times. It attempts to learn the user preferences using a MAB algorithm. If each of the multiple servers run their own algorithm, they waste the large amount of data which other servers collect. Or, if they communicate after every recommendation, the communication complexity becomes high within the servers themselves.

As observed from the example above, communicating after each time step is not favorable because of the increased communication cost. If $N$ agents communicate after every round to reduce the regret for $T$ time steps, their total regret is lower bounded by the regret of a super agent solving the MAB problem with $NT$ time steps. This bounds the total regret as at least $\Tilde{O}(\sqrt{NKT})$ or a per agent regret of $\Tilde{O}(\sqrt{KT/N})$. Whereas, if the $N$ agents interact with the MAB problem independently, without any information exchange with other agents, the individual regret bound is upper bounded by $\Tilde{O}(\sqrt{KT})$. We aim to find an algorithm which can obtain the regret bound of the super agent setup, $i.e.$, $\Tilde{O}(\sqrt{KT/N})$, though with limited communication between the agents.


We provide an algorithm, Limited Communication Collaboration - UCB, (\NAM), to minimize the regret. \NAM\ divides the arms among multiple agents, such that each agent only interacts with the MAB instance but plays arms only from a subset of all the arms. The algorithm proceeds in epochs which double in duration, where the agents use UCB algorithm to find the best arm in their smaller MAB problem and communicate at the end of each epoch. On receiving the messages from other agents, each agent updates its set of arms and restarts its algorithm. We prove the regret of \NAM\ is upper bounded by $\Tilde{O}\left(\sqrt{\left(K/N+N -1 \right)T}\right)$. For $N=1$, the regret of the \NAM\ algorithm reduces to the standard regret bounds of $\Tilde{O}(\sqrt{KT})$.

We also consider a general setup where the network of agents may not be completely connected and the agents may not be able to broadcast knowledge to all the other agents at once. Under such case, we propose \NAM-GRAPH algorithm that sub divides epochs into sub-epochs of equal length. The agents restart their UCB algorithm in each sub phases with the new information available from their neighbors. We show that the regret bound of this modified algorithm with divided phases changes to $\Tilde{O}\left(D\sqrt{\left(K/N+K_G\right)DT}\right)$, where $K_G$ is the maximum degree of the nodes in the graph. Also, the increased communication complexity of this algorithm is bounded by $O\left(K_G D\log T\right)$ message exchanges per node. The key novelty in both the algorithms is that the gap between the recommended arms and the optimal arm reduces with epochs.

Finally, we simulate and compare our algorithms with other communication protocols. 
We show that the algorithm behaves close to the communication strategy where the agents share the knowledge at each time step. For the \NAM-GRAPH algorithm we consider sparse graphs with more than $100$ nodes. We observe that the \NAM-GRPAH algorithm performs better than the communication strategy where the agents share local data with all their neighbors at every time step. Further, the \NAM\ and the \NAM-GRAPH algorithms also outperforms the DEMAB algorithm \cite{wang2020optimal} where agents communicate for only $O(N\log(NK))$ rounds.
\section{Related Works}

Optimal action selection problem dates back to \citep{thompson1933likelihood}, and since then many algorithm have been proposed and studied to solve the MAB problem ranging from index based policies \citep{gittins1979bandit}, Optimism in the Face of Uncertainty based UCB algorithm \citep{auer2002using,auer2010ucb,audibert2009minimax}, to Thompson Sampling algorithm \citep{agrawal2013further}. All the algorithms achieve a bound on regret $\Tilde{O}(\sqrt{KT})$ and match the lower bound of $\Omega(\sqrt{KT})$ upto logarithmic factors. 
Since then, various generalization and extensions have been proposed to solve various online learning problems using a bandit framework \citep{abbasi2011improved,li2010contextual,lattimore2018toprank,lale2019stochastic}. However, all these problems consider a single agent interacting with the environment.

Since the last decade, there has been a thrust in studying distributed agents solving an instance of MAB problems. \citet{kanade2012distributed} consider a model where $N$ agents talk to a central controller at every round. However, they considered the problem of reducing the communication cost for each agents connected in a star topology with a controller as the central node which is unlike our setup where we allow any topology, including central node/agent. \citet{hillel2013distributed} consider the problem of reducing communication cost for stochastic bandits in a setup where every agent can communicate with each other. Their work also bound the total communication rounds by $O(\log_2 T)$ using an action elimination based algorithm. However, their agents communicate the estimates of arm rewards for all the $K$ arms in each message, whereas, we bound the number of bits required in each message by $O(\log_2K)$. \citet{shahrampour2017multi} consider a setup where multiple agents collectively select an arm at a time step and observe different rewards sampled from different distribution for each agent.

Other works consider a setup where the agents talk to only one of the other nodes in a network at any given time step (gossiping style algorithm)~\citep{landgren2016distributed,martinez2019decentralized,wang2020optimal}. However, they allow their agents to communicate at every time step which is a different setup, and do not optimize a regret-communication trade-off. Further, they also send estimates of arm rewards in each message. \citet{sankararaman2019social,chawla20Gossiping} also consider a gossip style algorithms. Similar to us, these works divide the time horizon into epochs of variable length. Their strategies also divide the arms among the agents and the agents unicast the knowledge of the best arm they have using $O(\log K)$ bits in each epoch. However, because of gossip style communication protocols, an agent becomes aware of the best arm after it has already incurred $O(\frac{1}{\Delta^2})$ regret which translates to a problem independent bound of $\Tilde{O}(T^{2/3})$. We note that we use the same number of communication as these papers, while achieve better regret bound of $\Tilde{O}(T^{1/2})$. Further, we can convert the proposed broadcast based communication of our work to a unicast based strategy by sending a message to each neighbor at one timestep for $N$ timesteps.

\citet{wang2019distributed,dubey2020kernel,dubey2020differentially} consider the problem of distributed linear bandits. They considered a fully connected network for reducing the communication messages and reduce the average regret for $N$ agents. In contrast, we aim to find bounds on the regret of each of the $N$ agents for $K$-armed stochastic bandits. 

\citet{wang2019distributed} propose DEMAB algorithm for a distributed bandit setup where all the nodes communicates with a central node. The setup assumes knowledge of the time horizon to cleverly obtain a bound on number of communications messages that is independent of time. The DEMAB algorithm is based on action elimination that also proceeds in epochs with duration growing exponentially after an initial period of length $T/(NK)$ where every agent eliminates arms independently. In each epoch, the algorithm generates new estimates of arm rewards discarding the old samples. This results in high constants $O(\sqrt{2^{14}})$ in the regret term. The regret bounds of the proposed \NAM\ algorithm only exceeds the regret of DEMAB for $\log_2 T > 2^{14}/144$. Additionally, the DEMAB algorithm requires a central coordinating node, which may not always be the case. Lastly, for an unknown time horizon the number of messages increases back to $O(\log T)$ which is the same as ours. 

The proposed algorithm, \NAM, obtains $\Tilde{O}(\sqrt{(N/K)T})$ for each agent with messages of size $O(\log K)$ with a total of $O(\log T)$ messages, thus achieving the regret of $\Tilde{O}(\sqrt{T})$ 
Additionally, the proposed \NAM-GRAPH algorithm works well on sparse graphs with large number of agents with communication complexity of $O(D\log_2 T)$.
\section{Problem Formulation}
We consider a network of $N$ agents, indexed as $n\in [N] = \{1, 2, \cdots, N\}$.
Each agent $n\in[N]$ interacts with the same instance of stochastic $K$ armed bandit over $T$ time steps. An agent $n$, at time $t$, plays an arm $i_n(t)$. The expected reward of arm $i$ is $\mu_i$ for all $i\in[K]$. On playing an arm $i_n(t)\in[K]$ at time $t$, the $n$'th agent receives a reward of $r_{n,t} = \mu_{i_n(t)} + \eta_{n,t}$. We assume that $\eta_{n,t}$ is $1$-sub Gaussian at every time step $t$, i.e., for any $\lambda\in\mathbb{R}$, we have $\mathbb{E}[\exp{(\lambda\eta_{n,t})}]\leq \exp{(\lambda^2/2)},~\forall (n,t)\in[N]\times[T]$. For our analysis, we assume that $\mu_1 \geq \mu_2 \geq \cdots \geq \mu_K$. However, the ordering is unknown to the agents. We also define the gap between two arms as $\Delta_{i} \coloneqq \mu_1 - \mu_i$. For our analysis we assume $0 \leq \mu_i \leq 1\ \forall\ i\in[K]$.
For our system model, we assume that $N \ll K$ as observed in many practical setups. For example, an e-commerce website will have many more products listed than the number of servers deployed.

We assume that all the agents can communicate with each other (we later relax this assumption in Section \ref{sec:general_graph}). This implies, whenever an agent broadcasts a message, all the other $N-1$ agents receive the message. Further, we assume that each agent only communicates the index of the best arm it knows. This requires $\lceil\log(K)\rceil$ bits for every message and since there are $N-1$ other agents to send the message, the total bits required by any agent is $(N-1)\lceil\log(K)\rceil$ bits in every communication round.

An agent $n$ aims to minimize its cumulative regret over time $T$, $R_n(T)$, defined as:
\begin{align}
R_n(T) &= T\mu_1 - \mathbb{E}\left[\sum_{t=1}^T\sum_{i=1}^K\mu_{i}\bm{1}\{i_n(t) = i\}\right]
\end{align}
Note that minimizing regret $R_n(T)$ for all agents $n\in[N]$ also minimizes the total cumulative regret over the agents as well.

\section{\NAM\ Algorithm}

We design our algorithm \NAM\ on the basis of the fact that the regret of UCB algorithms \cite{auer2002using,bubeck2011pure,lattimore2020bandit} scales as $\Tilde{O}(\sqrt{KT})$. We reduce the per step regret by distributing the $K$ arms among the $N$ agents in growing in length epochs. An agent $n$ chooses to interact with a potentially smaller set of arms $\mathcal{S}_n$ where $\mathcal{S}_n = \{\left((n-1)\lceil\frac{K}{N}\rceil \mod K\right) + 1, \cdots, \left((n\lceil\frac{K}{N}\rceil -1) \mod K\right) + 1\}$. 
For the first epoch, i.e., $j=0$, each agent starts with possibly sub-optimal arms, even the worst possible arms. As the algorithm proceeds, in epoch $j\geq 1$, agents broadcast the most played arm by UCB algorithm during epoch $j$ to all the other agents. Each agent $n\in[N]$ receives $\mathcal{R}_{n,j}$, a set of arm recommendations from other $N-1$ agents. The agent now runs the UCB algorithm \cite{bubeck2011pure} over the arms in the augmented set $\mathcal{A}_{n,j} = \mathcal{S}_{n}\cup \mathcal{R}_{n,j}$. At the end of any epoch, the agent purges any old recommendations it has and starts again with the new recommendations received after an epoch. This ensures that the number of arms with any agent does not exceed $K'\coloneqq \lceil K/N\rceil + N-1$. This approach helps to bound the regret of any agent $n\in[N]$ by  $\Tilde{O}\left(\sqrt{\left(\lceil{K/N}\rceil + N-1\right)T}\right)$. 




The \NAM\ algorithm running at an agent $n\in[{N}]$ is described in Algorithm \ref{alg:collab}. The algorithm at agent $n$ receives the set of initial arms $\mathcal{S}_n$, the indices of other agents, and the total horizon $T$. The agent $n$ maintains a set $\mathcal{R}_n$ of the arms received from the remaining $[N]\setminus\{n\}$ agents. For the first epoch $\mathcal{R}_n = \emptyset$ as the agent has not heard anything from the remaining agents and the augmented set is same as the initial set of arms,  $\mathcal{A}_{n,0} = \mathcal{S}_n$. As the algorithm proceeds, it runs the UCB algorithm \citep{auer2002finite,bubeck2011pure}, described in Algorithm \ref{alg:UCB}, on the arms in the augmented set $\mathcal{A}_{n,j}$ for epoch duration $K'(K'+1)2^j$.
If at time $t$, remaining time is not sufficient to run a complete epoch of duration $T_j$, it just runs the UCB algorithm for the remaining horizon $T-t$.

\begin{algorithm}[!htb]
	\small
	\begin{algorithmic}[1]
        \STATE $t = 0, j = 0$
        \STATE $\mathcal{R}_{n,j} = \emptyset$
        \WHILE{$t < T$}
            \STATE Set augmented set $\mathcal{A}_{n,j} = \mathcal{S}_{n}\cup\mathcal{R}_{n,j}$
            \STATE $i^*$ = UCB($\mathcal{A}_{n,j}, \min\left(T-t, K'(K'+1)2^j\right)$)
            \STATE $t = t + K'(K'+1)2^j$
            \STATE $j = j + 1$
            \STATE Send $i^*$ to other $[N]\setminus\{n\}$ agents
            \STATE Receive most played arms of $[N]\setminus\{n\}$ agents as $\mathcal{R}_{n,j}$
        \ENDWHILE
	\end{algorithmic}
	\caption{{\NAM}{$\left(\mathcal{S}_n, [N]\setminus\{n\}, T\right)$}}\label{alg:collab}
\end{algorithm}

\begin{algorithm}[!htb]
	\begin{algorithmic}[1]
        \STATE $t_j = 0$
        \STATE $N_i(t_j) = 0, \hat{\mu}_i = 0, B_i = \infty,~\forall\ i \in \mathcal{A}$
        \FOR{$t_j = 1, \cdots, T_j$}
            \STATE Obtain reward $r_t$ by playing arm $i_t$, where
            $$i_{t_j} = \arg\max_{i\in \mathcal{A}}\left\{\hat{\mu}_i +\sqrt{\frac{2\log(t_j)}{N_i(t_j)}}\right\}$$
            \STATE $N_i(t_j) = N_i(t_j-1) + \bm{1}_{\{i_{t_j} = i\}}~\forall~i\in\cal A$
            \STATE Update $\hat{\mu}(i_t) = \frac{\hat{\mu}_i\times N_i(t_j-1) + r_t}{N_i(t_j)}$
        \ENDFOR
        \STATE Return $i^* = \arg\max_{i\in \mathcal{A}} N_i(T_j)$
	\end{algorithmic}
	\caption{{UCB}{$\left(\mathcal{A}, T_j\right)$}}\label{alg:UCB}
\end{algorithm}

\section{Main Result}
We now state the main result for bounding the regret and number of communications for the proposed \NAM\ algorithm.

\begin{theorem}\label{thm:main_threorem}
The regret of any agent $n$ following \NAM\ algorithm is bounded by
\begin{align}
    R_n(T) \leq O\left(\sqrt{K'T}\log(T)\right),
\end{align}
where $K' = \lceil K/N\rceil + N-1$.
\end{theorem}

To prove Theorem \ref{thm:main_threorem}, we first state the necessary lemmas required for the construction of the proof. Note that, the \NAM\ algorithm bounds regret when agent $1$ recommends an arm $i^*$ which is ``close'' to the best arm ($i=1$) from its augmented set $\mathcal{A}_{1,j}$ at every epoch, and then, in the following epoch, every other agent $n$ minimizes the regret with respect to the their augmented sets $\mathcal{A}_{n,j+1}$ which now contain the arm $i^*$.

Since, the agent runs UCB algorithm (Algorithm \ref{alg:UCB}) which returns the most played arm for each epoch. We now state and prove the lemma that the most played arms by the UCB algorithm is ``good'', or $\mu_{i^*}\geq \mu_1 - \Tilde{\Delta_j}$, with high probability for some $\Tilde{\Delta_j}$. 

\begin{lemma}\label{lem:UCB_max_play_arm}
For any epoch $j$, such that $T_j \geq K'(K'+1)$, instance of the UCB Algorithm \ref{alg:UCB} running at agent $1$ returns an arm $i^*$ that satisfies $\mu_{i^*}\geq \mu_i - \Tilde{\Delta}_{j}$, with probability atleast
\begin{align}
    1-K'\left(\frac{T_j}{K'}-1\right)^{-2},
\end{align}
for $\Tilde{\Delta}_{j} = \sqrt{\frac{16K'\log T}{T_j}}$.
\end{lemma}
\begin{proof}
	We first note that the augmented set at agent $1$ contains the best arm $1$ as arm index $\left((n-1)\lceil\frac{K}{N}\rceil \mod K\right) + 1\in\mathcal{S}_n$. From Algorithm \ref{alg:UCB} instance that ran at epoch $j$, $N_i(T_j)$ is the number of times arm $i\in \mathcal{A}_{1,j}$ is played in epoch $j$. 
	We now prove that the arm $i_{n,j}^* = \arg\max_{i\in\mathcal{A}_{1,j}}N_{i}(T_j)$ is at most $\Tilde{\Delta}_{j}$ far from the true optimal arm $1$.
	
	For time step $t_j$ in epoch $j$, we construct an event where arm $i$ is selected and the total plays $N_i(t_j -1)$ of arm $i$ has exceeded some number $l$ as ${\cal G}_{t_j}(i) = \left\{\{i_{t_j} = i\} \cap \{N_i(t_j-1) \geq l_i\}\right\}$ for $l_i = 1+ \frac{8\log T}{\Delta_i^2}$ and $t_j \geq K+1$ as each arm is played atleast once. Then, the Theorem 1 of \cite{auer2002finite} states that the probability of the event ${\cal G}_{t_j}(i)$ is upper bounded by $2t_j^{-3}$. Using the probability of the event ${\cal G}_{t_j}(i)$, we can bound the probability of the event that the number of plays of an arm exceeds $l_i$ by using union bound. Specifically we have:
	\begin{align}
		Pr\left(N_i(T_j) \geq l_i \right) &\leq \bigcup_{t_j = l_i}^{T_j}Pr\left({\cal G}_{t_j}(i)\right)\\
		&\leq \sum_{t_j = l_i}^{T_j} 2t_j^{-3}\\
		&< \sum_{t_j = l_i}^{\infty} 2t_j^{-3}\\
		&\leq \int_{t_j = l_i-1}^{\infty} 2t_j^{-3} = \frac{1}{(l_i-1)^2}\label{eq:prob_total_play_exceed_bound}
	\end{align}
	
	Now, for an arms $i$ such that $\Delta_i > \sqrt{\frac{8K'\log T}{T_j -K' }} \geq \sqrt{\frac{16K'\log T}{T_j}}\eqqcolon \Tilde{\Delta}_{j}$, we have,
	\begin{align}
		l_i &= 1+\frac{8\log T}{\Delta_i^2}\\
		&< 1+\frac{8\log T}{\Tilde{\Delta}_j^2}\\
		&\leq 1+\frac{T_j -K'}{K'} = \frac{T_j}{K'}. \label{eq:max_play_sub_opt}
	\end{align}
	Hence, $N_i(T_j) \leq T_j/K' - 1$ with probability at least $1-(T_j/K - 1)^2$. Further, taking a union bound over all arms, we obtain that for any arm $i$, such that $\Delta_i \geq \Tilde{\Delta}_{j}$, $N_i(T_j) \leq T_j/K'$ with probability at least $1-K'(T_j/K - 1)^2$.
	
	After bounding the number of plays of arms $i$, such that $\mu_i \leq \mu_1 - \Tilde{\Delta}_{j}$, with high probability, we show that the most played arm $i^*$ has expected reward $\mu_{i^*}\geq \mu_1 - \Tilde{\Delta}_{j}$. Let $\mathcal{B}_{1,j} = \{i\in\mathcal{A}_{1,j}|\mu_i <\mu_{i_{n,j}^*}  - \Tilde{\Delta}_j\}$ be the set of ``bad'' arms in the augmented set of agent $n$ in epoch $j$. We have:
	\begin{align}
		\max_{i\in\mathcal{A}_{1,j}\setminus\mathcal{B}_{1,j}} N_{i}(T_j) &\geq \frac{1}{|\mathcal{A}_{1,j}\setminus\mathcal{B}_{1,j}|}\sum_{i\in\mathcal{A}_{1,j}\setminus\mathcal{B}_{1,j}} N_{i}(T_j)\\
		&= \frac{1}{|\mathcal{A}_{1,j}\setminus\mathcal{B}_{1,j}|}\left(T_j - \sum_{i\in\mathcal{B}_{1,j}}N_{i}(T_j)\right)\\
		&> \frac{1}{|\mathcal{A}_{1,j}\setminus\mathcal{B}_{1,j}|}\left( T_j -  \sum_{i\in\mathcal{B}_{1,j}}\left(\frac{T_j}{K'}\right)\right)\\
		&= \frac{1}{|\mathcal{A}_{1,j}\setminus\mathcal{B}_{1,j}|}\left(K'\frac{T_j}{K'} -  |\mathcal{B}_{1,j}|\left(\frac{T_j}{K'}\right)\right)\\
		&= \frac{1}{|\mathcal{A}_{1,j}\setminus\mathcal{B}_{1,j}|}\left(K'-|\mathcal{B}_{1,j}|\right)\frac{T_j}{K}\\
		&= \frac{1}{|\mathcal{A}_{1,j}\setminus\mathcal{B}_{1,j}|}|\mathcal{A}_{1,j}\setminus\mathcal{B}_{1,j}|\frac{T_j}{K} = \frac{T_j}{K}
	\end{align}
	
	This proves that the most played arm in $\mathcal{A}_{1,j}$, $i_{n,j}^*$, is at most $\Tilde{\Delta}_j$ far from the optimal arm $1$.
\end{proof}

After showing that the agent $1$ returns a good arm after each epoch, we now show that the regret of all the other agents is bounded in the following epoch $j+1$. Lemma \ref{lem:UCB_regret_bound} bounds the regret of an agent $n$ running UCB Algorithm \ref{alg:UCB} during an epoch $j$. We then sum over all the epochs to obtain the total regret of the algorithm. We focus our analysis on an agent $n$. The analysis of the remaining agents follows identically.


\begin{lemma}[UCB regret bound]\label{lem:UCB_regret_bound}
The regret of any agent $n$ running UCB algorithm described in Algorithm \ref{alg:UCB} for an epoch $j\geq 2$ with $T_j$ time steps is upper bounded by
\begin{align}
    R(T_j) &\leq 6\sqrt{2K'T_j\log T} + \frac{16K'^3}{T_j} + 2K'\label{eq:epoch_regret}
\end{align}
\end{lemma}
\begin{proof}
We first consider the case of an agent $n\neq 1$. The agent receives recommendations from all the other $N-1$ agents including the agent $1$ and hence contains the arm $i^*$ recommended by the agent $1$. 

To analyze the regret, we first create some events that will help in analysis. The first event denotes the case where the agent $1$, after the end of epoch $j-1$, recommends arm $i^*$ such that $\mu_{i^*} \geq \mu_1-\Tilde{\Delta}_{j-1}$. We denote this event as $\Tilde{\cal G}_1$. Further note that $N_{i}(T_j)$ is the number of times agent plays arm $i\in \mathcal{A}_{n,j}$ in epoch $j$. We note that when the event $\Tilde{\mathcal{G}}_1$ occurs $\Delta_{i^*}\leq \Tilde{\Delta}_{j-1}$. We assume that $i^*$ satisfies $\mu_{i^*} = \max_{i\in\mathcal{A}_{n,j}}\mu_i$. In case the assumption is not valid, we redefine $i^*$ as $i^* = \arg\max_{i\in\mathcal{A}_{n,j}}\mu_i$, and we still have $\mu_1-\mu_{i^*}\leq \Tilde{\Delta}_j$. Also, for the simplicity of notation, we define $\Delta_{i^*,i} = \mu_{i^*} - \mu_i$.  Then, using the regret decomposition lemma (Lemma 4.5) from \cite{lattimore2020bandit}, the regret of the UCB algorithm for epoch $j$ is upper bounded as:
\begin{align}
	R(T_j) &= \sum_{i\in\mathcal{A}_{n,j}}\mathbb{E}\left[\Delta_iN_i(T_j)\right]\\
	&= \sum_{i\in\mathcal{A}_{n,j}}\mathbb{E}\left[(\mu_1 - \mu_i)N_i(T_j)\right]\\
	&= \sum_{i\in\mathcal{A}_{n,j}}\mathbb{E}\left[(\mu_1 -\mu_{i^*} + \mu_{i^*} - \mu_i)N_i(T_j)\right]\\
	&= \sum_{i\in\mathcal{A}_{n,j}}\mathbb{E}\left[(\Delta_{i^*} + \Delta_{i^*,i})N_i(T_j)\right]\\
	&= \sum_{i\in\mathcal{A}_{n,j}}\mathbb{E}\left[\Delta_{i^*}N_i(T_j)\right] +\sum_{i\in\mathcal{A}_{n,j}}\mathbb{E}\left[ (\Delta_{i^*,i})N_i(T_j)\right]\\
	&= \sum_{i\in\mathcal{A}_{n,j}}\mathbb{E}\left[\Delta_{i^*}N_i(T_j)|\Tilde{\mathcal{G}}_1\right]Pr(\Tilde{\mathcal{G}}_1) + \sum_{i\in\mathcal{A}_{n,j}}\mathbb{E}\left[\Delta_{i^*}N_i(T_j)|\Tilde{\mathcal{G}}_1^c\right]Pr(\Tilde{\mathcal{G}}_1^c)+\sum_{i\in\mathcal{A}_{n,j}}\mathbb{E}\left[ (\Delta_{i^*,i})N_i(T_j)\right]\\
	&\leq \sum_{i\in\mathcal{A}_{n,j}}\Tilde{\Delta}_{j-1}\mathbb{E}\left[N_i(T_j)|\Tilde{\mathcal{G}}_1\right]Pr(\Tilde{\mathcal{G}}_1) + \sum_{i\in\mathcal{A}_{n,j}}\mathbb{E}\left[N_i(T_j)|\Tilde{\mathcal{G}}_1^c\right]Pr(\Tilde{\mathcal{G}}_1^c)+\sum_{i\in\mathcal{A}_{n,j}}\mathbb{E}\left[ (\Delta_{i^*,i})N_i(T_j)\right]\\
	&\leq \Tilde{\Delta}_{j-1}\mathbb{E}\left[\sum_{i\in\mathcal{A}_{n,j}}N_i(T_j)|\Tilde{\mathcal{G}}_1\right] + Pr(\Tilde{\mathcal{G}}_1^c)\mathbb{E}\left[\sum_{i\in\mathcal{A}_{n,j}}N_i(T_j)|\Tilde{\mathcal{G}}_1^c\right]+\sum_{i\in\mathcal{A}_{n,j}}\mathbb{E}\left[ (\Delta_{i^*,i})N_i(T_j)\right]\\
	&\leq \Tilde{\Delta}_{j-1}T_j + K'\left(\frac{K'}{T_{j-1}-K'}\right)^2T_j+\sum_{i\in\mathcal{A}_{n,j}}\mathbb{E}\left[ (\Delta_{i^*,i})N_i(T_j)\right]\\
	&\leq 4\sqrt{\frac{K'\log T}{T_{j-1}}}T_j + K'\left(\frac{2K'}{T_{j-1}}\right)^2T_j+\sum_{i\in\mathcal{A}_{n,j}}\mathbb{E}\left[ (\Delta_{i^*,i})N_i(T_j)\right]\\
	&\leq 4\sqrt{2K'T_j\log T} + \frac{16K'^3}{T_{j}}+\sum_{i\in\mathcal{A}_{n,j}}\mathbb{E}\left[ (\Delta_{i*,i})N_i(T_j)\right]\label{eq:regret_bound_1}
\end{align}

We now focus on the last term. We define event where the UCB algorithm plays arm $i$ after the number of plays of an arm $i$ is has crossed $l_i$, or
\begin{align}
	{\cal G}_{n,i}(t_j) = \left\{\{i_t = i\}\cap\{N_i(t_j - 1) \geq l_i\}\right\}\text{,   where } l_i = \frac{1}{\Delta_{i^*,i}},
\end{align}
Again, similar to Lemma \ref{lem:UCB_max_play_arm}, we use the Theorem 1 of \cite{auer2002finite} to upper bound the probability of the event ${\cal G}_{t_j}(i)$ by $2t_j^{-3}$. Then we can bound the last term in Equation \ref{eq:regret_bound_1} as:
\begin{align}
	\sum_{i\in\mathcal{A}_{n,j}}\mathbb{E}\left[ (\Delta_{i*,i})N_i(T_j)\right] &\leq \sum_{i\in\mathcal{A}_{n,j}}\Delta_{i*,i}l_i + \sum_{i\in\mathcal{A}_{n,j}}\sum_{t_j=l_i}^{T_j}Pr\left({\cal G}_{n,i}(t_j)\right)\\
	&\leq \sum_{i\in\mathcal{A}_{n.j}}\Delta_{i^*,i}\left(1+\frac{8\log T}{\Delta_{i^*,i}^2}\right) + \sum_{i\in\mathcal{A}_{n.j}}\sum_{t_j = 1}^{T_j}t_j^{-2}\\
	&\leq \sum_{i\in\mathcal{A}_{n.j}}\frac{8\log T}{\Delta_{i^*,i}} + K' + \frac{K'\pi}{6}\\
	&\leq \sqrt{8K'T_j\log T}+ K' + \frac{K'\pi}{6}
\end{align}
Replacing the value in Equation \ref{eq:regret_bound_1}, we get the required result for $n\neq 1$.

Further, note that for $n=1$, the true optimal arm $1$ is always present in $\mathcal{A}_{1,j}$ for all $j\geq 1$. 
\end{proof}

We are now ready to prove Theorem \ref{thm:main_threorem}. We first note that for epoch $j = 0$, not agents have yet communicated, and hence the regret of any agent is trivially bounded by $T_0 = K'(K'+1)$. For the later epochs, we sum over the regret incurred in each epoch using Lemma \ref{lem:UCB_regret_bound}. To do so, we first bound the total number of epochs. Let the total number of epochs be $J$, then noting that the total number of time steps is $T$, we get:
\begin{align*}
    &T \leq \sum_{j=0}^{J-1}K'(K'+1)2^j < 2T\label{eq:num_epochs_lower_bound}\\
    \implies &2^{J} -1 < \frac{2T}{K'(K'+1)}\\
    \implies &J < \log_2\left(\frac{T}{K'(K'+1)} + 1\right)\\
    \implies &J = \lfloor \log_2\left(\frac{T}{K'(K'+1)}+1\right)\rfloor
\end{align*}

After bounding the regret in each epoch $R(T_j)$ and bounding the total number of epochs, 
we can bound the total regret as,
\begin{align}
    R_n(T) &= \sum_{j=0}^{J-1}R(T_j)\\
           &\leq 6\sum_{j = 1}^{J-1}\sqrt{2K'T_j\log T} + \sum_{j = 1}^{J-1}\frac{16K'^3}{T_j} + 2K'\log_2(2T+1) + K'(K'+1)\\
     &\leq 6\sqrt{J\sum_{j = 1}^{J-1}2K'T_j\log T} + 16K' + 2K'\log_2(2T+1) + K'(K'+1)\label{eq:Cauchy_Schwarz}\\
           &\leq 6\sqrt{2K'\log_2(2T+1)\log T(2T)} + 16K'+ 2K'\log_2(2T+1) + K'(K'+1)\\
           &\leq 12\sqrt{K'T(\log_2(2T+1))\log T} + 16K' + 2K'\log_2(2T+1) + K'(K'+1),
\end{align}
where Equation \eqref{eq:Cauchy_Schwarz} follows from the Cauchy Schwarz inequality.

\begin{theorem}
For \NAM\ algorithm, total number of bits exchanged by an agent is bounded by $O\left(N\log(K)\log(T)\right)$.
\end{theorem}
\begin{proof}
An agent sends or receives only arm index, which requires $\log_2(K)$ bits. In each epoch, the agents communicates with $N-1$ agents and sends and receives $2(N-1)\log_2(K)$ bits. Finally, there are $\log_2(T)$ epochs. This bounds the total number of bits as $O\left(K\log(K)\log(T)\right)$.
\end{proof}

We note that the algorithm proposed by \cite{sankararaman2019social} also divides the time horizon into epochs with $K$ arms divided among $N$ agents. However, they consider the first few epochs to be of fixed length where agents only explore to find the best arm within themselves. Our algorithm runs UCB from the very first epoch. Also, the length of the first epoch is $o(1)$ in \NAM\ algorithm which limits the regret. These novel changes allow for a significantly improved regret bound as compared to the state of the art with limited communications.

\section{Extension to general network structures}\label{sec:general_graph}
So far we assumed that all the nodes are connected to each and every other node. However, this might not always be true. We now assume a general structure where a graph $G = (V, E)$ that has the different agents as vertices and the connections as edges represents the network structure. We assume that the graph representing the network is sparsely connected with a small diameter and degree for example Erd\H{o}s-R\'{e}nyi graphs \cite{chung2001diameter}. We assume that the maximum degree of $G$ is $K_G$ and the diameter of $G$ is $D$. 

For this setup, we assume that an agent or node can communicate with only its neighbors. Under this assumption, it may take multiple epochs for the knowledge of the best arm to reach an agent that may not have the best arm to begin with. Further, the number of epochs where an agent does not hear from the agent that has the best arm is bounded by the diameter $D$. Also, the maximum size of $\mathcal{G}_n$ is now upper bounded by $\lceil\frac{K}{N}\rceil + K_G$ instead of $\lceil\frac{K}{N}\rceil + N-1$. 

We first start with a direct extension of the result in Theorem \ref{thm:main_threorem}, and by understanding the issues in the direct extension, will propose an algorithm to improve the results for general networks. The following result gives a corollary for Theorem \ref{thm:main_threorem} for general graphs. 
\begin{corollary}\label{col:exp_regret_LCC_UCB}
For graph $G = (V, E)$ with agents as nodes $V$, \NAM\ algorithm results in a regret bound of:
\begin{align}
&R_n(T) \leq \Tilde{\mathcal{O}}\left(2^{D}K'^2 +\sqrt{D(2^{D})K'T}\right)
\end{align}
where $D$ is the diameter of the graph $G$, $K' = \left(\Big\lceil \frac{K}{N}\Big\rceil + K_G\right)$ and $K_G$ is the maximum degree of any node in the graph $G$.
\end{corollary}
\begin{proof}An agent $n\neq 1$ receives arm recommendations only from its neighboring nodes which results in reduction of $K'$ from $\lceil K/N\rceil + N$ to $\lceil K/N\rceil + K_G$. However, this also implies that the $n\neq 1$ does not obtains information about a good arm from the agent $1$ directly. Note that applying Lemma \ref{lem:UCB_max_play_arm} on UCB algorithm ran by agent $n\neq 1$ suggests that the agent recommends an arm $i_n^*$ such that $\mu_{i_n^*} \geq \mu_{i^*}-\Tilde{\Delta}_j$ where $i^* = \arg\max_{i\in\mathcal{A}_{n,j}}\mu_i$. This implies that the agent (or node) located farthest from the agent $1$ receives knowledge about a good arm, \textbf{(1)} only after $D$ epochs for the very first time, and \textbf{(2)} the best arm in the received $i^*=\arg\max_{i\in\mathcal{R}_{n,j}}\mu_{i}$ set $i^*$ satisfies $\mu_i^* \geq \mu_1 - \sum_{j=1}^{D}\Tilde{\Delta}_{j-1}$.
	
	This results in an additional constant regret during the first $D$ epochs as:
	\begin{align}
		\sum_{j=0}^{D-1}T_j = \sum_{j=0}^{D-1}(K'+1)K'2^j = (K'+1)K'(2^D - 1)
	\end{align}
	
	Further, the gap incurred from receiving a bad recommendation in each epoch scales as:
	\begin{align}
		(\mu_1 - \mu_{i^*})T_j &\leq \sum_{j'=1}^{D}\Tilde{\Delta}_{j'-1}T_j = \sum_{j'=1}^D4\sqrt{\frac{K'\log T}{T_{j'-1}}}T_j\\
		&= \sum_{j'=1}^D4\sqrt{K'2^{j'}T_j\log T}\\
		&= 4\sqrt{D\sum_{j'=1}^DK'2^{j'}T_j\log T}\\
		&= 4\sqrt{DK'2(2^{D}-1)T_j\log T}
	\end{align}
\end{proof}
\begin{remark}
	Note that for $D=1$ and $K_G = N-1$, or the case for a completely connected graph, the result of Theorem \ref{thm:main_threorem} is obtained.
\end{remark}
To avoid the exponential blow-up of $2^D$ in the regret, we first consider a strategy where an agent forwards the messages from one neighbor to all the other neighbor. However, this increases the message size from $O(K_G\log_2 K)$ bits to $O(N\log_2 K)$ bits. Further, additional complexity is added to reduce repeated propagation of messages. In order to avoid the potential exponential increase in regret or increase in the message size and the communication complexity, we propose a modification of the \NAM\  algorithm as \NAM-GRAPH algorithm. The proposed \NAM-GRAPH algorithm is described in Algorithm \ref{alg:graph_collab}.

\begin{algorithm}[!htb]
	\small
	\begin{algorithmic}[1]
        \STATE $t = 0, j = 0$
        \STATE $\mathcal{R}_{n,1,0} = \emptyset$
        \FOR{$t < T$}
            \STATE $d = 1$
            \FOR {$d \le D$}
            \STATE Set augmented set $\mathcal{A}_{n,d,j} = \mathcal{S}_n\cup\mathcal{R}_{n,d,j}$
            \STATE $i^*$ = UCB($\mathcal{A}_{n,d,j}, \min(T - t, K'(K'+1)2^j)$)
            \STATE $t = t + K'(K'+1)2^j$
            \STATE Send $i^*$ to neighbors
            \STATE Receive most played arms of neighbors as $\mathcal{R}_{n,d,j}$
            \STATE $d = d+1$
            \ENDFOR
            \STATE $j = j + 1$
        \ENDFOR
	\end{algorithmic}
	\caption{{\NAM-GRAPH}{$\left(\mathcal{S}_n, G, T_0, T\right)$}}\label{alg:graph_collab}
\end{algorithm}

The \NAM-GRAPH algorithm further divides an epoch $j$ into $D$ sub-epochs indexed as $d$. The duration of each sub-epoch in epoch $j$ is $T_j = K'(K'+1)2^j$. Now, the \NAM-GRAPH algorithm restarts UCB algorithm for sub epochs (Line 6-12). Additionally, the agents now communicate after every sub-epoch, but, only with their neighbors. This gives the $K' \leq \lceil\frac{K}{N}\rceil + K_G$.

Note that results from sub-epoch $d$ of epoch $j$ are propagated throughout the graph by the time sub-epoch $d$ starts in epoch $j+1$. Hence, for $\Tilde{\Delta}_j\coloneqq \sqrt{\frac{16K'\log(T)}{T_j}}$, this approach allows to propagate arms with $\Delta_i \leq D\Tilde{\Delta}_{j-1}$ instead of $\sum_{j'=j-D}^j\Tilde{\Delta}_{j'}$. Based on this modification, we can bound the regret of \NAM-GRAPH algorithm and the number of bits required for communications by \NAM-GRAPH algorithm.

\begin{theorem}\label{thm:graph_regret_theorem}
Let $G = (V, E)$ be the graph representing the network structure of agents $n\in [N]$, and let $D$ be the diameter of the graph $G$ and let $K_G$ be the maximum degree of the vertices of the graph $G$. Then, the regret of any agent $n$ following \NAM-GRAPH algorithm is bounded by
\begin{align}
    R_n(T) \leq \Tilde{O}\left(D\sqrt{DK'T}\right),
\end{align}
where $K' = \lceil\frac{K}{N}\rceil + K_G$.
\end{theorem}
\begin{proof}
Note that at the beginning of the phase of a sub-epoch $d$ in epoch $j$, the information from the farthest node $D$ edges away is also received for epoch $j-1$ sub-epoch $d$. This is because exactly $D$ communications happens between sub-epoch, epoch pair $d, j-1$ and $d, j$. Further, each intermediate $D$ nodes drifts from the optimal arm found in sub-epoch, epoch $d, j-1$ by at most $\Tilde{\Delta}_{j-1}$. This suggest that instead of receiving an arm with $\Delta_i \leq \Tilde{\Delta}_{j-1}$, the node actually receives an arm $i^* = \arg\max_{i\in\mathcal{A}_{n,d,j}}\mu_i$ with $\Delta_{i^*} \leq D\Tilde{\Delta}_{j-1}$. Hence, extending Lemma \ref{lem:UCB_regret_bound} with $D$ hops, the regret $R(d,j)$ in each sub-epoch $d$ and epoch $j$ is now upper bounded as
\begin{align}
    R(d,j) &\leq 2(2D+1)\sqrt{2K'T_j\log T} + \frac{16DK'^3}{T_j} + 2K'\label{eq:sub_epoch_regret_bound}
\end{align}
In Equation \eqref{eq:sub_epoch_regret_bound}, the extra factors of $D$ comes from the fact that now each of the agents in $D$ hops recommends an arm $i$ such that $\mu_{i_d^*} \geq \mu_{i_{d-1}^*}-\Tilde{\Delta}_j$ for all $d\geq 1$ and $i_0^* = 1$, the true best arm. Note that the duration of any sub-epoch $d$ is $K'(K'+1)2^j$ and it depends only on the epoch $j$. Hence, the regret $R(d, j)$ is only a function of epoch count $j$.

The total regret of the agent $n$, which is the sum of regrets over all sub-epochs in every epoch, can now be bounded as:
\begin{align}
R_n(T) &= \sum_{j=0}^{J-1}\sum_{d=1}^DR(d,j)\nonumber\\
       &= \sum_{j=1}^{J-1}\sum_{d=1}^DR(d,j) + \sum_{d=1}^DR(d,0)\nonumber\\
       &= \sum_{j=1}^{J-1}\sum_{d=1}^D\Big(2(2D+1)\sqrt{2K'T_j\log T}\nonumber\\
       &~~~ + \frac{16DK'^3}{T_j} + 2K'\Big) + \sum_{d=1}^DK'(K'+1)\\
       &= 2(2D+1)\sqrt{DJ\sum_{d=1}^D\sum_{j=0}^{J-1}2K'T_j\log T}\nonumber\\
       &~~~+ DJ\frac{16DK'^3}{T_j} + 2DJK'+ DK'(K'+1) \nonumber\\
       &= 4(2D+1)\sqrt{K'DT(\log_2(2T+1))\log T}\nonumber\\
       &~~~ + 16D^2K' + 2K'D\log_2(2T+1) + DK'(K'+1)\nonumber
\end{align}
\end{proof}

The key novelty of \NAM-GRAPH algorithm is to let sub-epochs $0\leq d< D$ collect the messages from the entire graph. The equal length of each sub-epoch avoids the exponential blow-up in the regret. Further, the exponential length of each epoch $j$ still keeps the total messages in logarithmic order of $T$.

\begin{theorem}\label{thm:graph_comm_theorem}
For \NAM\-GRAPH algorithm, total number of bits exchanged by an agent is bounded by $O\left(K_GD\log(K)\log(T)\right)$.
\end{theorem}
\begin{proof}
An agent sends or receives only arm index, which requires $\log_2(K)$ bits. The agent communicates at the end of every sub-epoch of every epoch. In each communications, the agents talks to at most $K_G$ neighbors and sends and receives $2K_G\log_2(K)$ bits. Finally, there are $D$ sub-epochs in every $\log_2(T)$ epochs. This bounds the total number of bits as $O\left(DK_G\log(K)\log(T)\right)$.
\end{proof}

Results from Theorem \ref{thm:graph_regret_theorem} and Theorem \ref{thm:graph_comm_theorem} suggest that it is possible to reduce the regret from an exponential order of the diameter $D$ of the graph $G$ at the expense of $D$ times more communication rounds. Further, since each communication involves only exchange of arm indices, the cost of communications is not high ($O(K_G\log_2 K)$ bits) for power constrained devices such as sensor networks. 
\begin{figure*}[ht]
    \centering
    \begin{subfigure}[b]{0.32\textwidth}
        \centering
        \includegraphics[width=\textwidth]{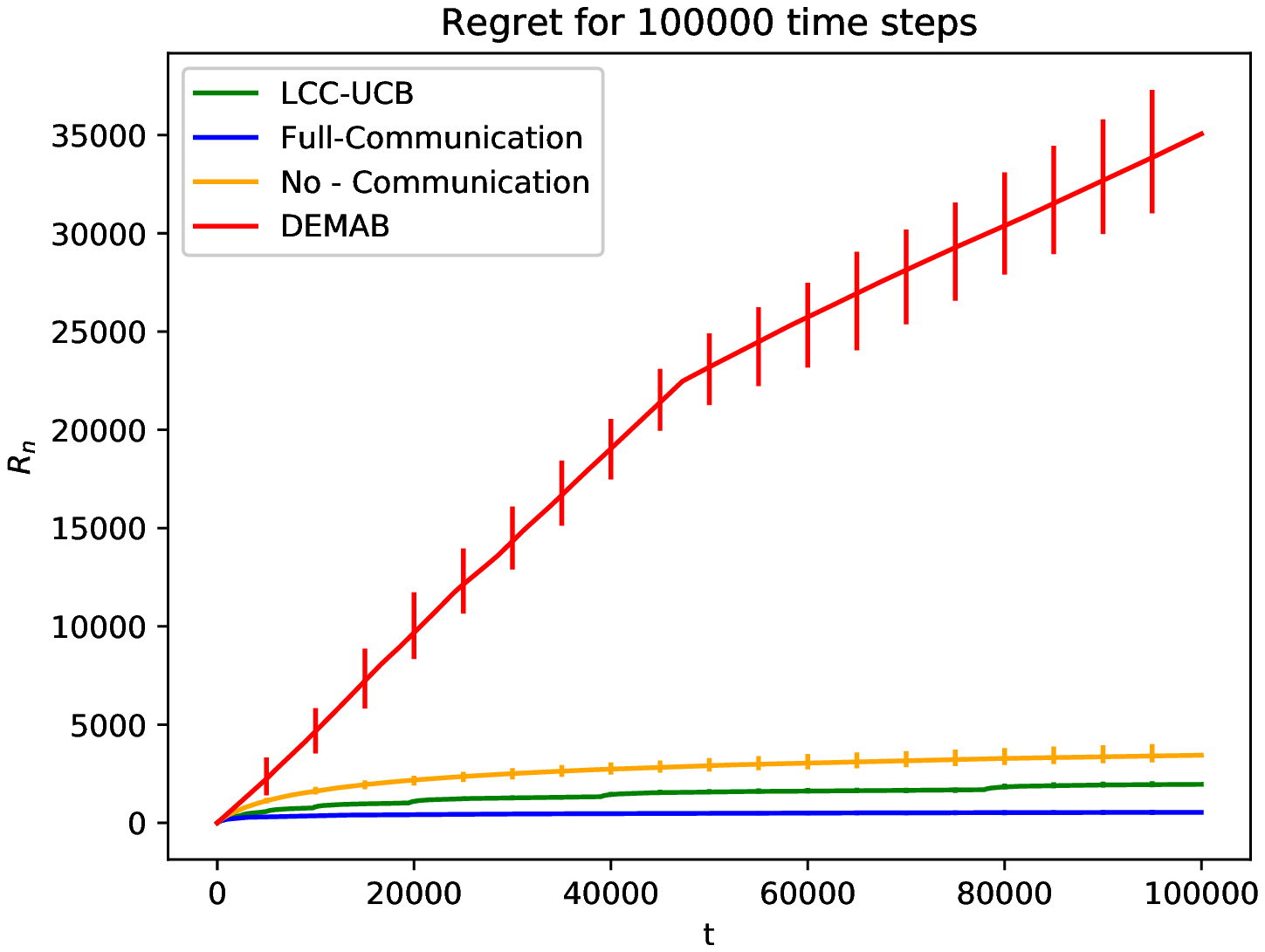}
        \caption{$(N, K)=(10, 100)$}
        \label{fig:fc_10_100_demab}
    \end{subfigure}
    \begin{subfigure}[b]{0.32\textwidth}
        \centering
        \includegraphics[width=\textwidth]{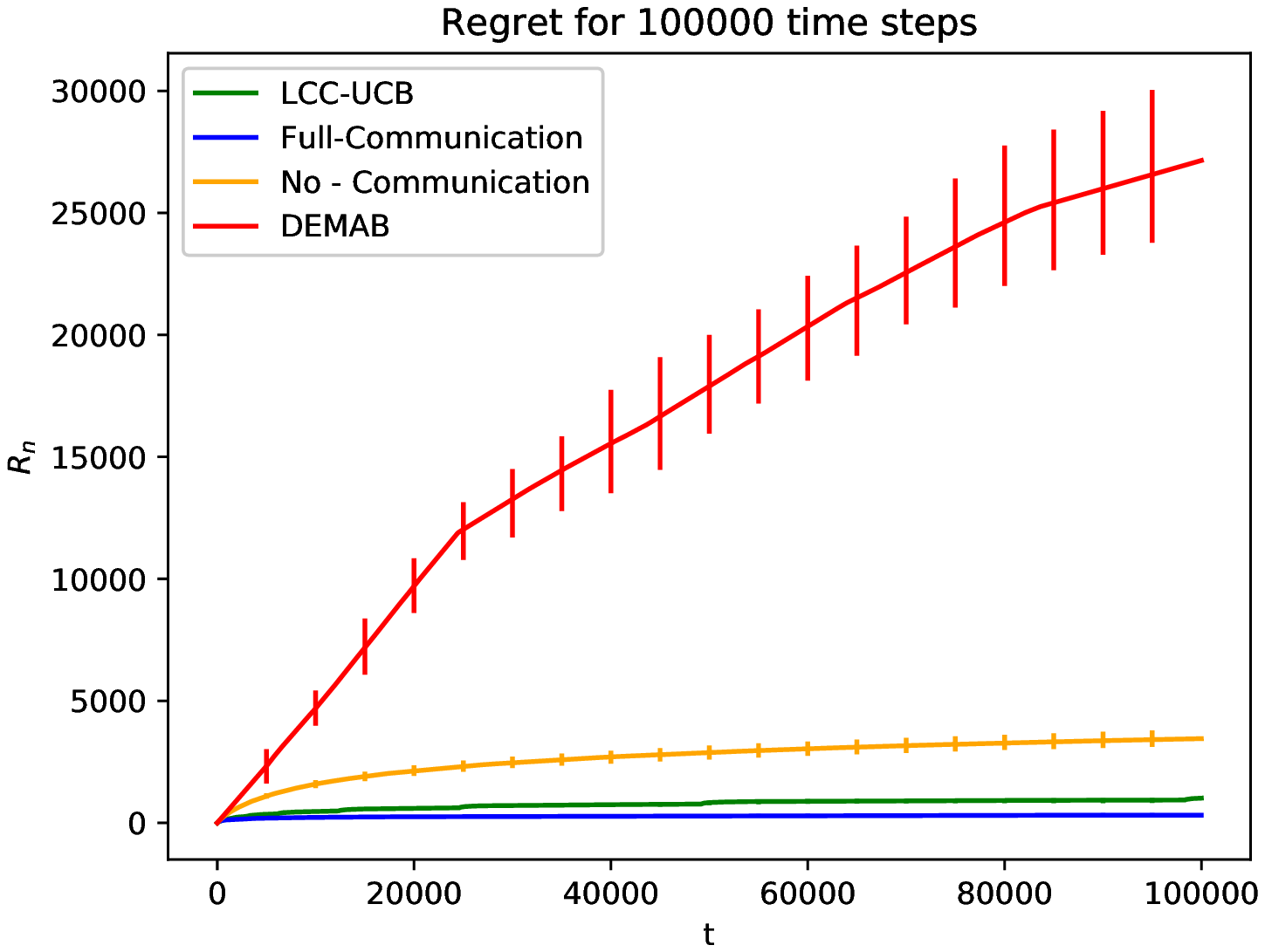}
        \caption{$(N, K)=(20, 100)$}
        \label{fig:fc_20_100_demab}
    \end{subfigure}
    \begin{subfigure}[b]{0.32\textwidth}
        \centering
        \includegraphics[width=\textwidth]{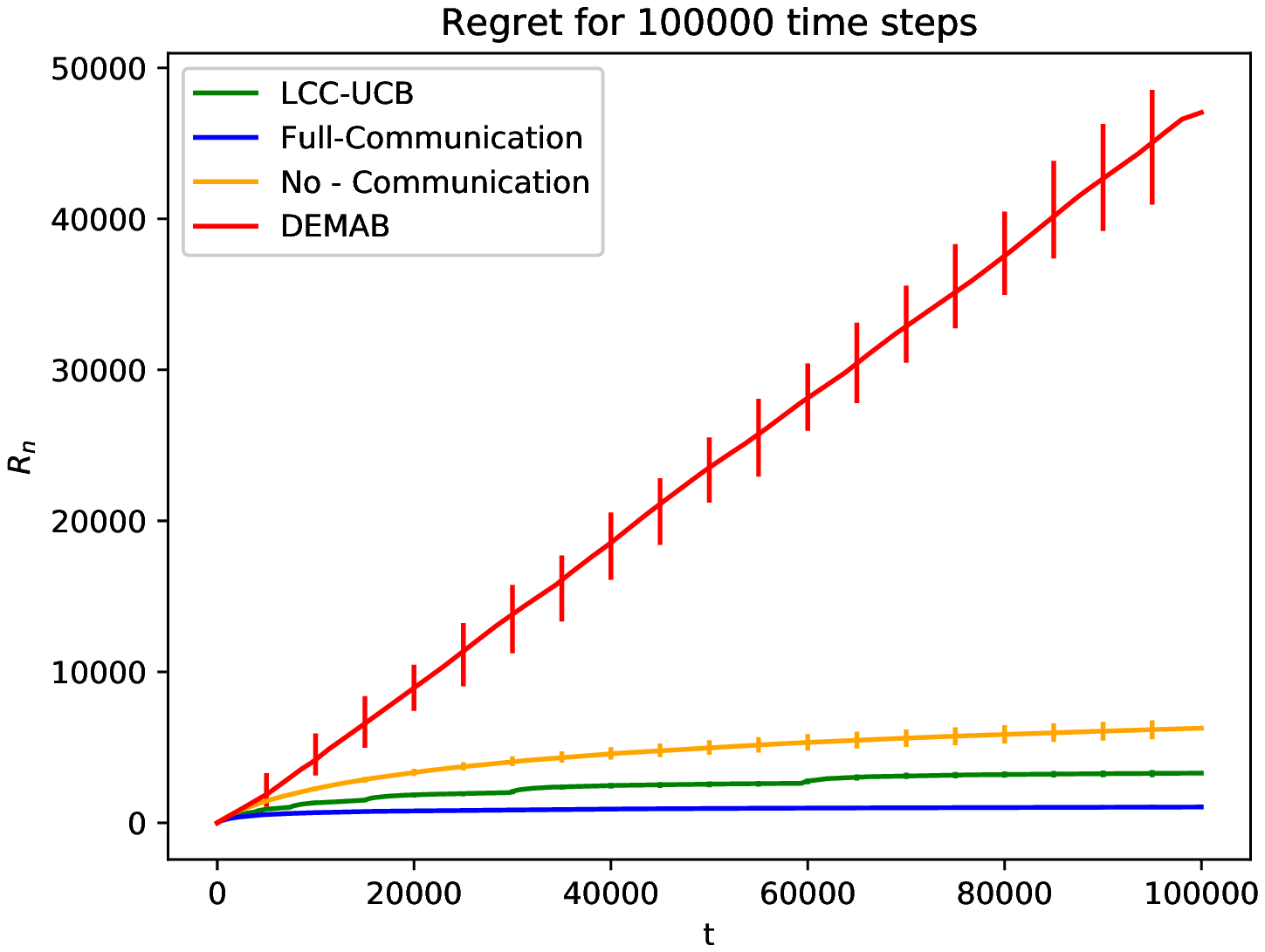}
        \caption{$(N, K)=(10, 200)$}
        \label{fig:fc_10_200_demab}
    \end{subfigure}
    \caption{Per-step cumulative regret for a single agent under various communication strategies.}
    \label{fig:per_step_regret_plot_with_DEMAB}
\end{figure*}

\begin{figure*}[ht]
    \centering
    \begin{subfigure}[b]{0.32\textwidth}
        \centering
        \includegraphics[width=\textwidth]{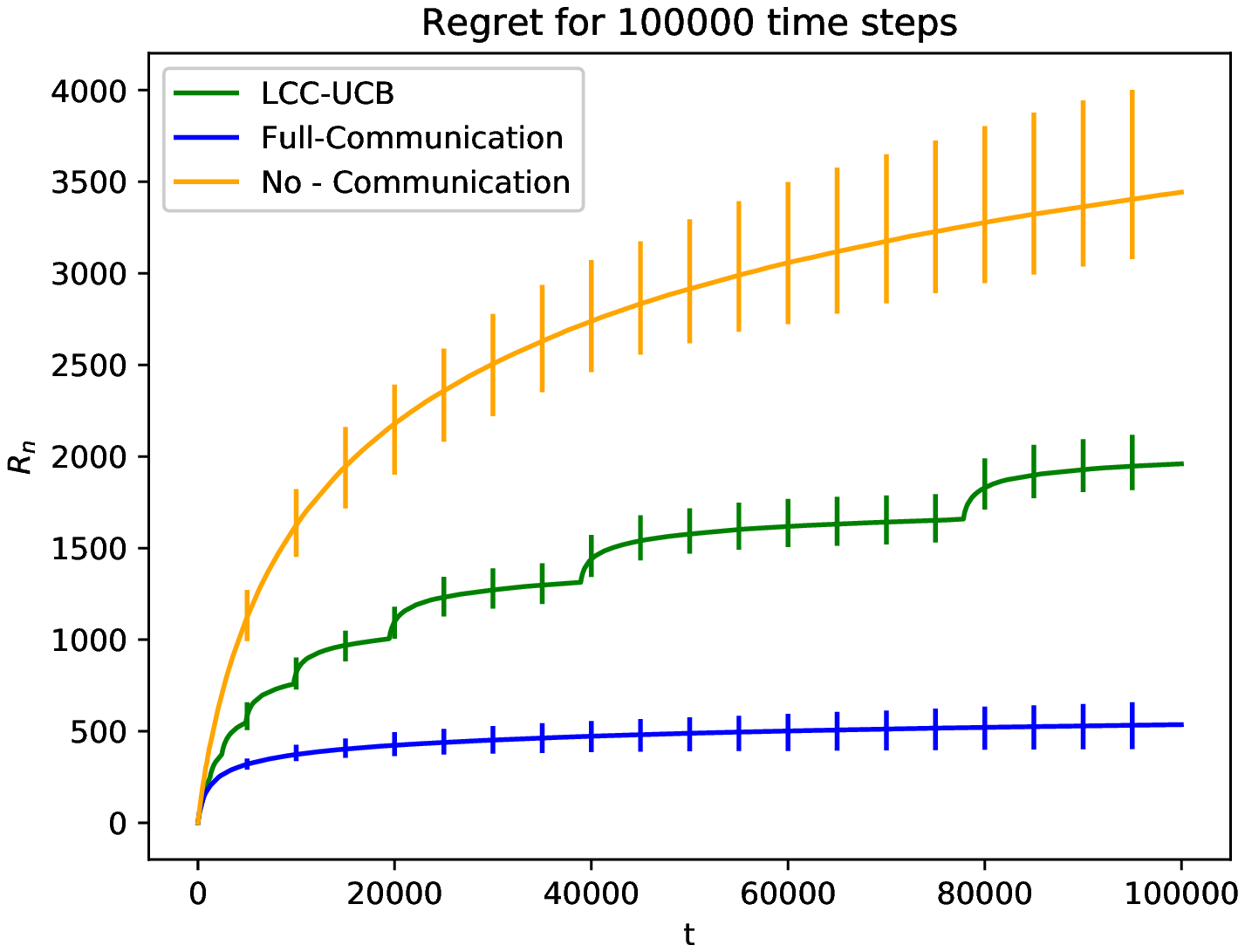}
        \caption{$(N, K)=(10, 100)$}
        \label{fig:fc_10_100}
    \end{subfigure}
    \begin{subfigure}[b]{0.32\textwidth}
        \centering
        \includegraphics[width=\textwidth]{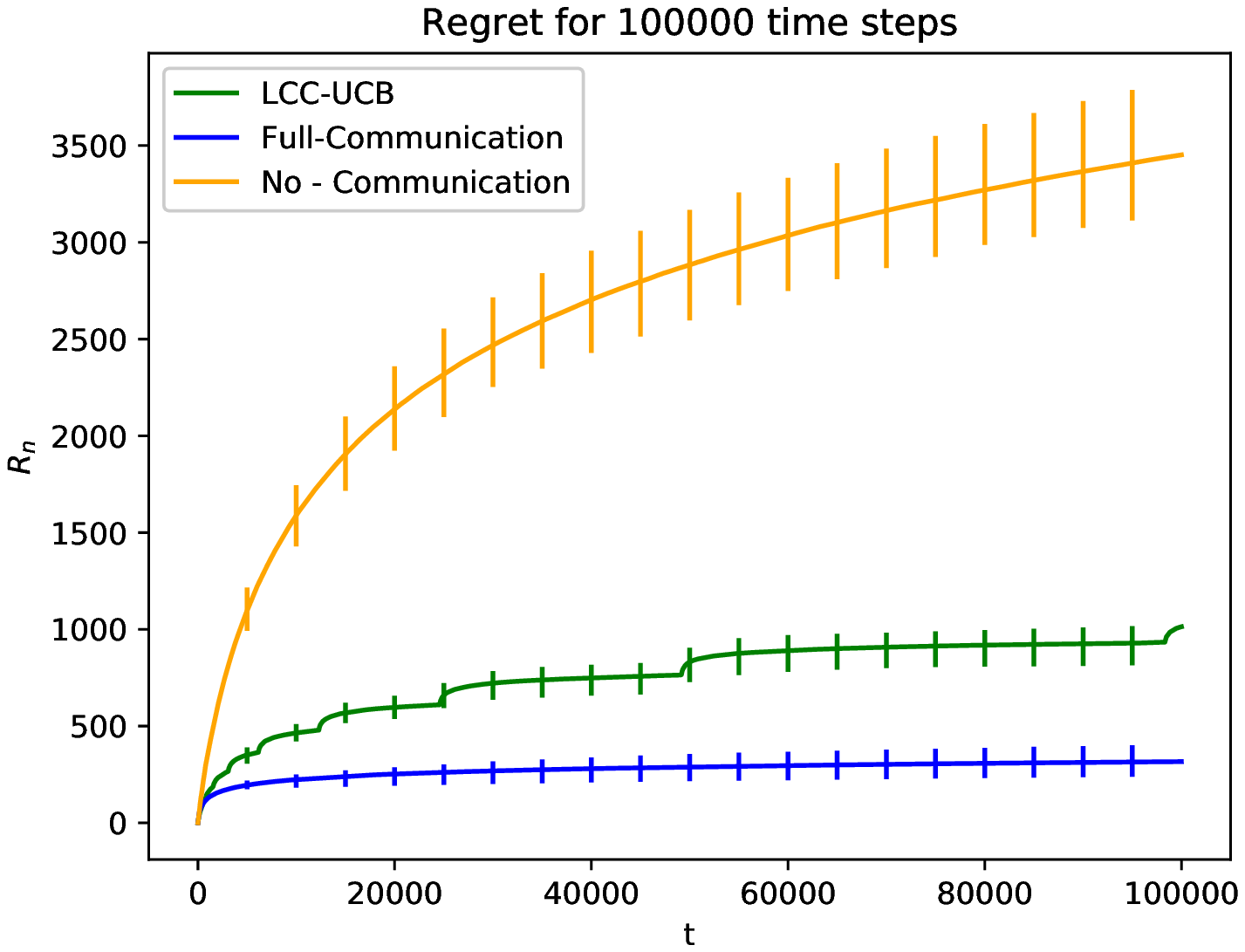}
        \caption{$(N, K)=(20, 100)$}
        \label{fig:fc_20_100}
    \end{subfigure}
    \begin{subfigure}[b]{0.32\textwidth}
        \centering
        \includegraphics[width=\textwidth]{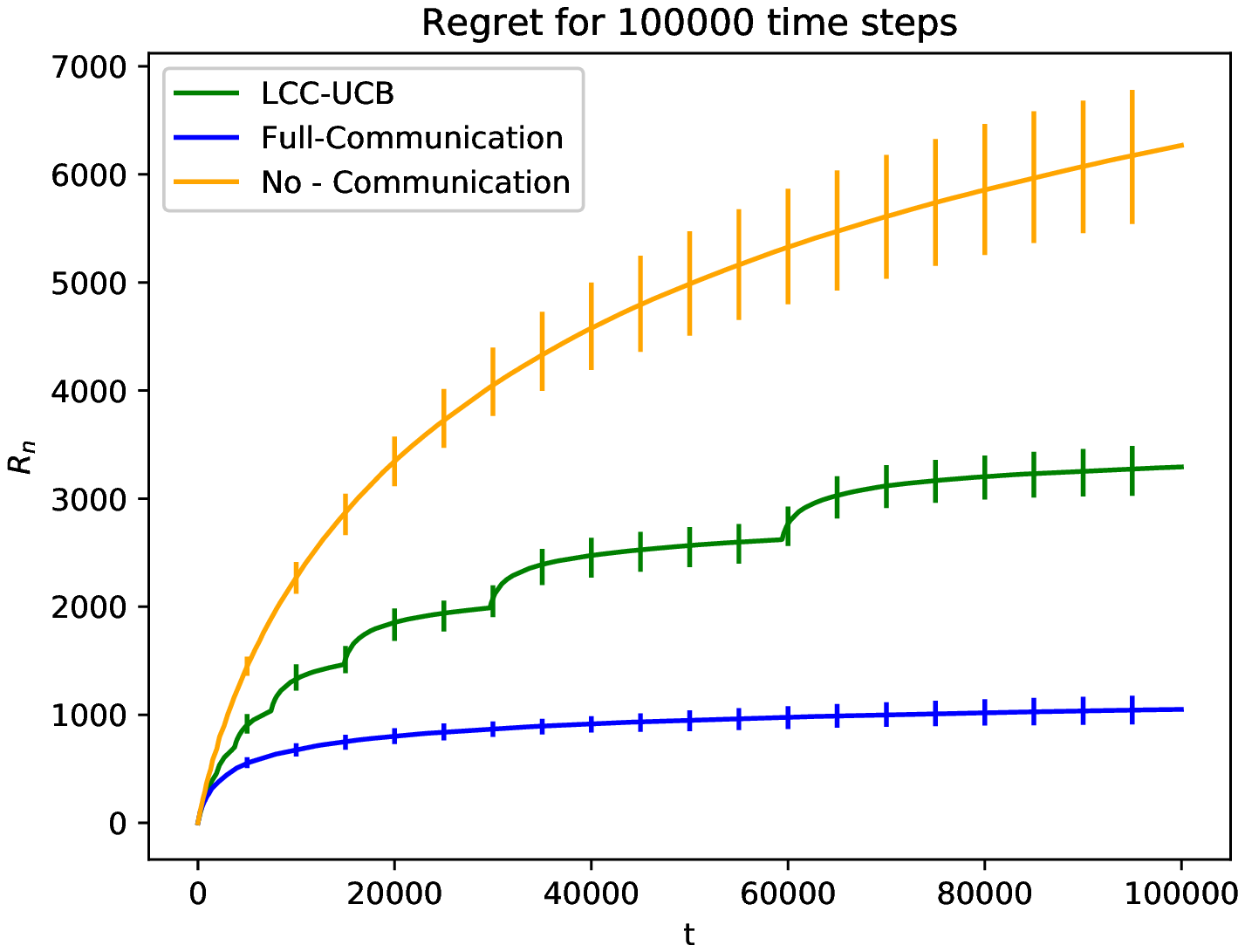}
        \caption{$(N, K)=(10, 200)$}
        \label{fig:fc_10_200}
    \end{subfigure}
    \caption{Per-step cumulative regret for a single agent under various communication strategies. (Excluding plots from DEMAB algorithm to the regret growth of other algorithms)}
    \label{fig:per_step_regret_plot}
\end{figure*}
\begin{figure*}[th]
    \centering
    \begin{subfigure}[b]{0.32\textwidth}
        \centering
        \includegraphics[width=\textwidth]{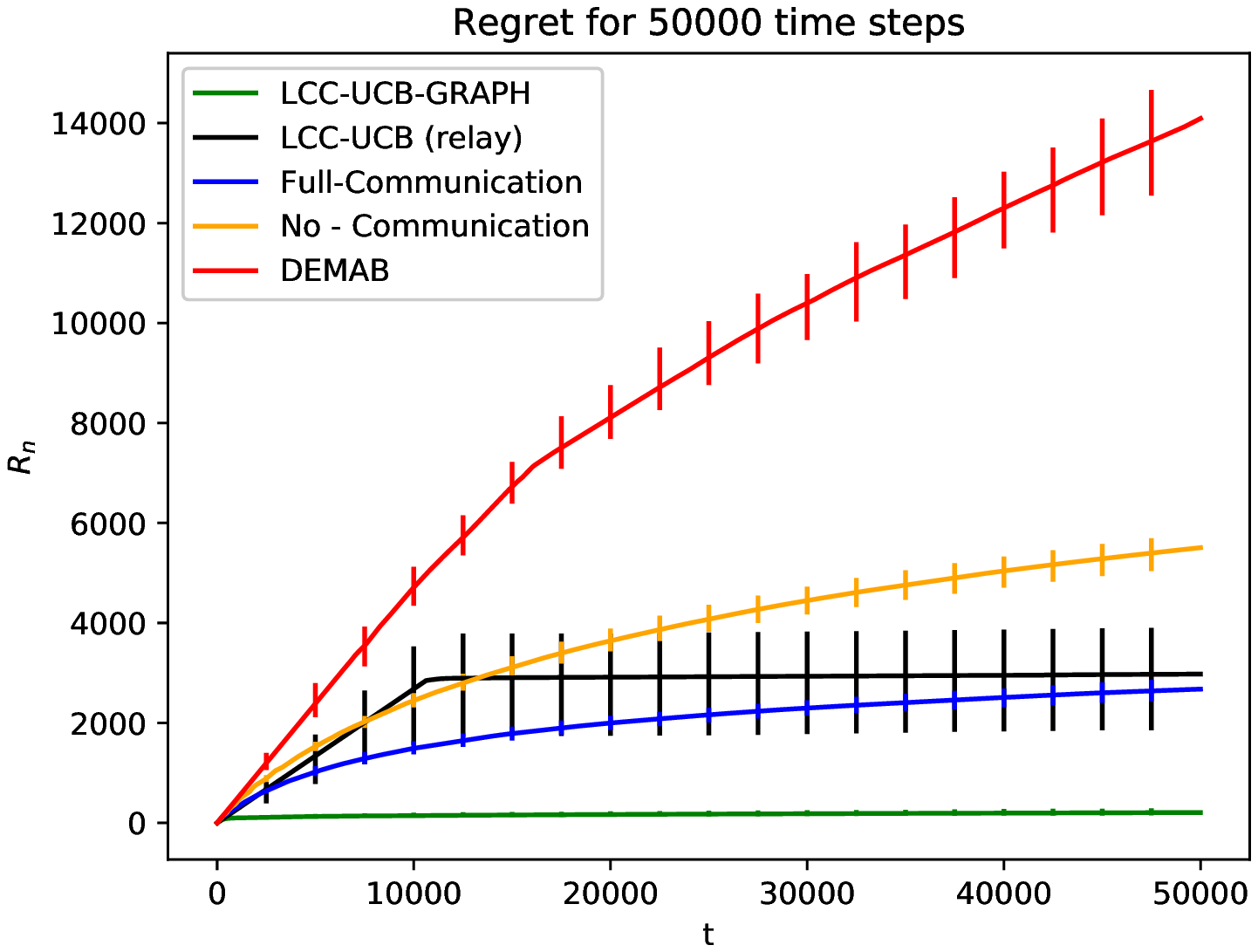}
        \caption{$(N, K)=(100, 250)$}
        \label{fig:sc_100_150}
    \end{subfigure}
    \begin{subfigure}[b]{0.32\textwidth}
        \centering
        \includegraphics[width=\textwidth]{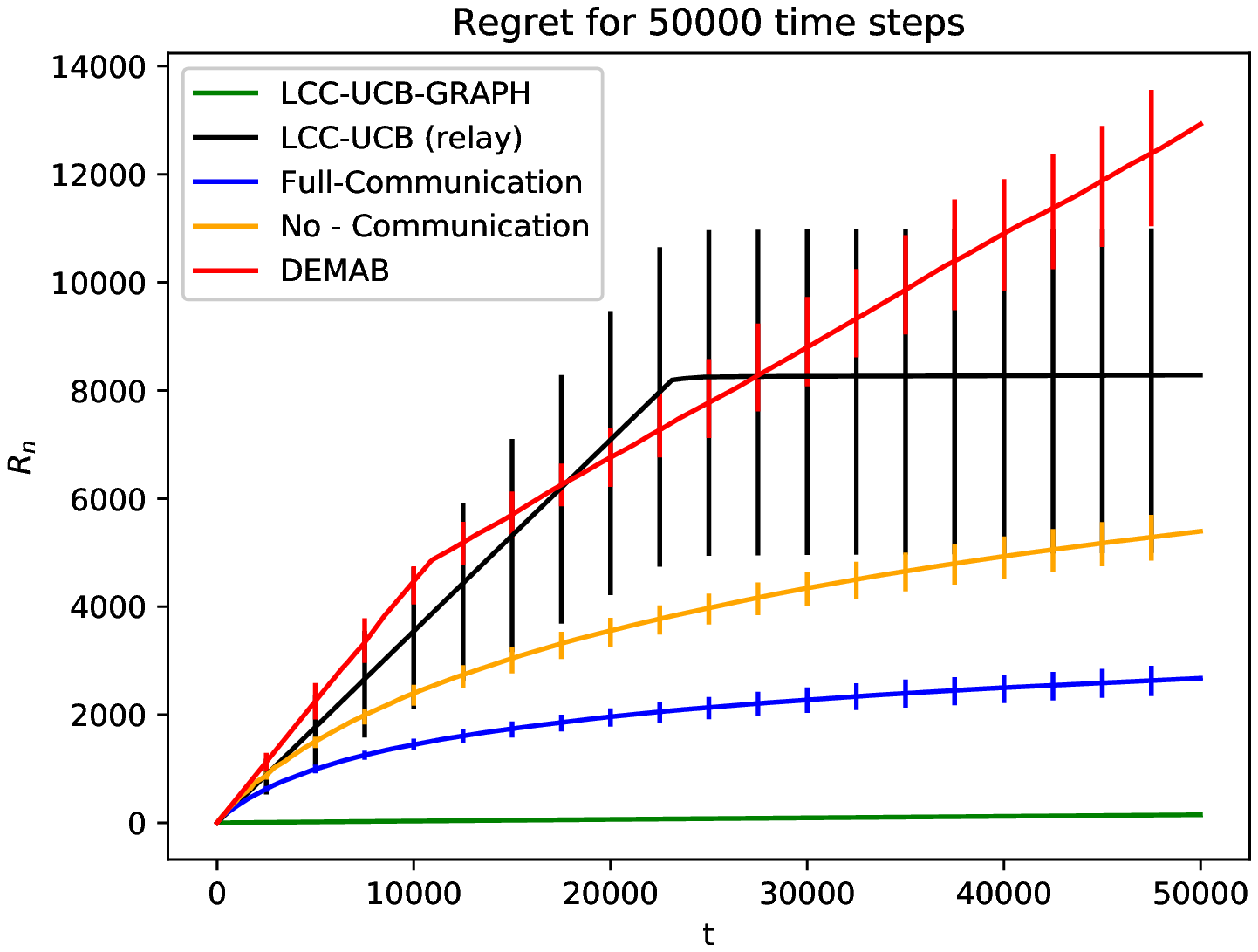}
        \caption{$(N, K)=(150, 250)$}
        \label{fig:sc_150_250}
    \end{subfigure}
    \begin{subfigure}[b]{0.32\textwidth}
        \centering
        \includegraphics[width=\textwidth]{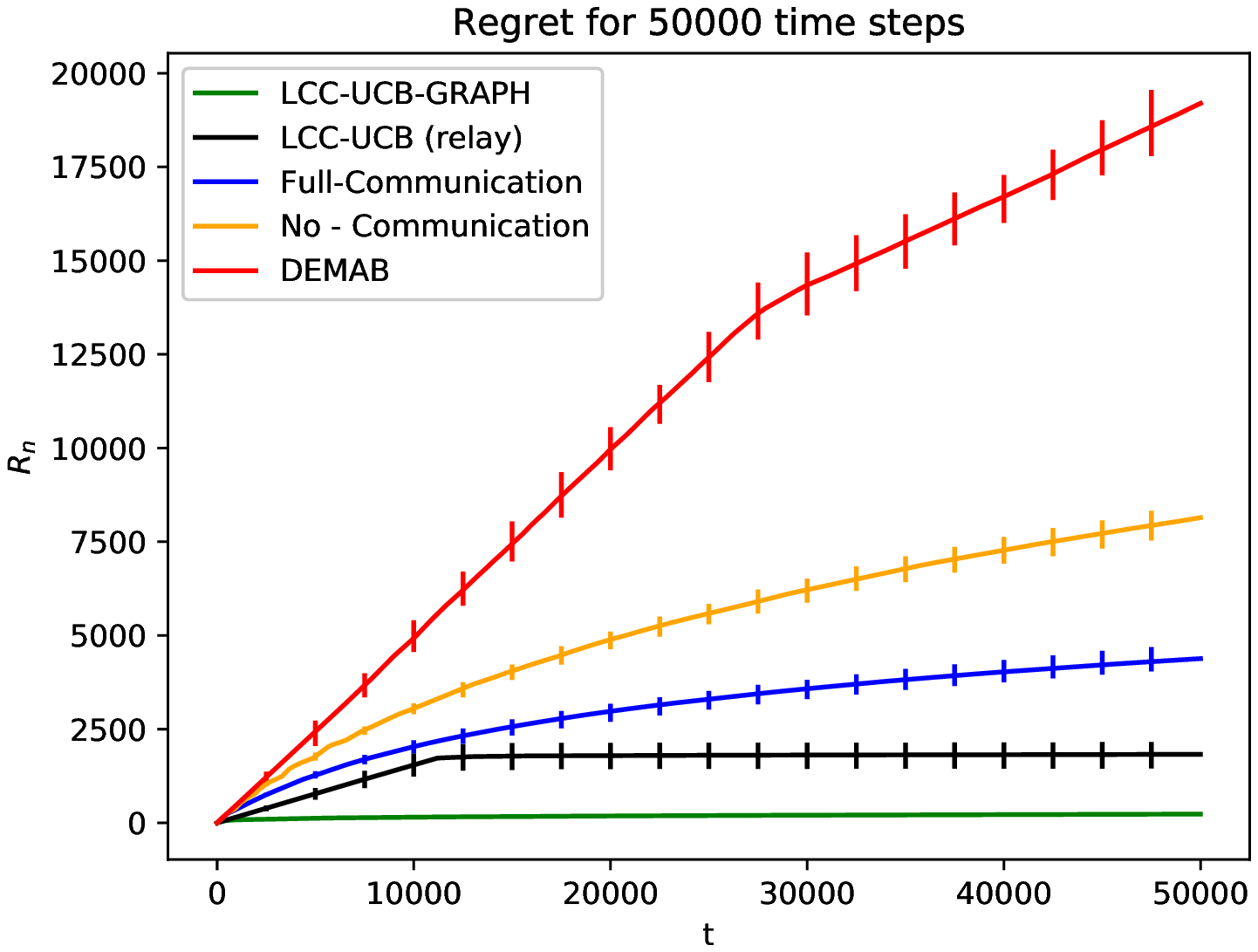}
        \caption{$(N, K)=(100, 500)$}
        \label{fig:sc_100_500}
    \end{subfigure}
    \caption{Per-step cumulative regret for a single agent, in a sparse graph, under various communication strategies.}
    \label{fig:per_step_regret_plot_graph}
\end{figure*}
\section{Evaluations}

We consider various problem setups to evaluate our algorithms. We compare with the setting where agents can communicate with their neighbors every time and with the setting where agents do not communicate with anyone for the entire time horizon. We also compare with the DEMAB algorithm, proposed by \cite{wang2020optimal}, which requires only $O(M\log (MK)$ communication rounds for known time horizons.

We first present the comparison results for Algorithm \ref{alg:collab}. We consider a horizon of $T = 10^5$ steps. We study the behaviour of the algorithm by varying the number of agents $N$ and the number of arms $K$. We choose three pairs $(N, K)$, which are $(10, 100)$, $(20, 100)$, $(10, 200)$. We present the result in Fig. \ref{fig:per_step_regret_plot_with_DEMAB} for $30$ independent runs for expected rewards drawn from uniform $\mathbb{U}(0,1)$ distribution. We plot the median of the cumulative regret incurred by a single angle at each time step and the $95\%$ confidence intervals. 

We first note that the regret of the DEMAB algorithm is even larger than the no-communication strategy. The high regret in the DEMAB algorithm is expected because the algorithm purges the observations collected after each epoch. Further, the agents do not share the knowledge of the best arm and continue to redivide the remaining arms to quickly eliminate the bad arms, and hence not all agents are able to exploit the best arm. This results in the high regret of the algorithm.
To show the scale between the remaining communication strategies, we plot the regret curves with the DEMAB algorithm in Figure \ref{fig:per_step_regret_plot}.


The start of an epoch $j$ can be observed as the jumps in the cumulative regret. We observe that the initial epochs incur the largest regret despite the duration $T_j$ being small. This is because the agents are not aware of the best arm yet and are exploring from possibly worst arms. Also, the regret grows very slowly in the later phase because most agents send the same arm index (the optimal arm) and the effective regret in the later rounds increase only as $\Tilde{O}\left(\sqrt{\left(1 + \lceil N/K \rceil\right)T_j}\right)$, instead of the upper bound of $\Tilde{O}\left(\sqrt{\left(N-1 + \lceil N/K \rceil\right)T_j}\right)$. We note that for small number of agents $N$ compared to the number of arms $K$, $(N, K) = (10, 100)$ and $(N, K)=(10, 200)$, the algorithm performs closer to the optimal case where the agents could communicate with each other as observed from Fig. \ref{fig:fc_10_100} and Fig. \ref{fig:fc_10_200}. This is because of the reduced overhead of re-sampling new arms obtained from all the agents.

We now evaluate the proposed \NAM-GRAPH algorithm on sparse graphs. We specifically consider Erd\H{o}s-R\'{e}nyi graphs $G(N, p)$ where $N \geq 100$ vertices are a swarm of $N$ agent. Also, $p = 10/N\geq \ln{N}/N$ is the edge selection probability. This gives an expected number of total edges in the graph to be $5N$. We consider only connected graphs (If the resulting graph is not connected, we sample another graph.). Once initiated, the graph does not changes structure over the subsequent time steps. This setup is typically used in placement of IoT devices communicating with only neighbors \cite{avner2016multi,sankararaman2019social}.

We again consider $3$ cases of  $(N, K)$ that are $(100, 250)$, $(150, 250)$, and $(100, 500)$. We present the result in Fig. \ref{fig:per_step_regret_plot_graph} for $30$ independent runs. Along with the expected rewards of the arms, graph structure is also different for each run. We plot the median of the cumulative regret incurred by a single angle at each time step and the $95\%$ confidence intervals. 

We note that for $K=250$, the performance is similar for $N=100$ (Fig. \ref{fig:sc_150_250}) and $N=150$ (Fig. \ref{fig:sc_150_250}). This is expected for no-communication strategy as the number of arms are same. For \NAM-GRAPH algorithm, this makes sense as the degree of the graph $K_G$ is higher than the the number of arms allocated to every agent $\lceil K/N\rceil$. For full communication strategy, this happens because the expected degree of each agent is same for both graphs. Each agent can access data from only neighbors, and that remains same. On doubling $K$ from $250$ to $500$, we observe that the regret increases at lower rate for \NAM-GRAPH than for the other two strategies. This is again attributed to the fact that $K_G$ dominates $\lceil K/N\rceil$ term in regret. We note that the performance of the DEMAB algorithm is still sub-par to the all the other three strategies. Note that the \NAM-GRAPH algorithm accumulates extremely low regret because of the reduced arms per agent $(\leq 5)$ and the degree of any node is also very low as we considered sparse $G(N, p)$ graphs with $p = 10/N$.

As expected, we note that the proposed strategy performs better than the no communication strategy. Further, we note that the proposed strategy even outperforms the strategy where communication happens after every time step and lags behind in initial time steps only. This is because an agent only shares what it knows with its neighbors and thus is not able to fully utilize the graph with $N$ agents. For the initial time steps, the \NAM-GRAPH algorithm performs pure exploration, hence incurs regret.

We also compare the performance of the \NAM-GRAPH algorithm against a modified \NAM\ algorithm which relays messages from other nodes. This modification allows every agent to receive recommendations from all the other agents after every epoch. However, the performance of the \NAM-GRAPH algorithm is significantly better than the relay based \NAM\ algorithm which justifies the sub-epoch based modification used in \NAM-GRAPH. \NAM\ algorithm wastes a significant portion of the time to generate good recommendations and hence incur a large regret. The better performance of the \NAM-GRAPH algorithm is because after each epoch, an agent only receives arm updates from its neighbors, and hence, the $\sqrt{K/N + K_G}$ term in regret is very small.

\section{Conclusion}
We considered the problem of reducing communications between $N$ agents and minimizing the regret of agents interacting with an instance of a Multi Armed Bandit problem with $K$ arms for time horizon $T$. We proposed two algorithm \NAM\ for fully connected networks and \NAM-GRAPH for sparse networks with maximum degree $K_G$ and diameter $D$. We analyzed the algorithms and obtain regret bound of $\Tilde{O}(\sqrt{T(N+K/N)})$ and $\Tilde{O}(D\sqrt{D(K/N + K_G)T})$ for \NAM\ and \NAM-GRAPH algorithms respectively. 
We found that the algorithms perform well empirically with the \NAM-GRAPH algorithm outperforming every time communication strategy in which an agent shares knowledge only with its neighbors. 
Further, both the \NAM\ and the \NAM-GRAPH algorithm beat the existing state of the art results. 
Additionally, the low bit complexity for communication in both the algorithms makes them a suitable choice for power constrained devices.

\bibliographystyle{apalike}
\bibliography{refs.bib}
\end{document}